\documentclass[sigconf]{acmart}

\AtBeginDocument{%
  }
\usepackage[ruled,vlined]{algorithm2e}
\usepackage{multirow}
\usepackage{tabularx}
\usepackage{cleveref}
\setcopyright{acmlicensed}
\copyrightyear{2024}
\acmYear{2024}
\acmDOI{XXXXXXX.XXXXXXX}

\acmConference[KDD 2024 Workshop]{2nd Workshop on Causal Inference and Machine Learning in Practice}{August 25--26,
    2024}{Barcelona, Spain}
\acmISBN{978-1-4503-XXXX-X/18/06}

\begin{document}

\title{Practical Marketplace Optimization at Uber Using Causally-Informed Machine Learning}


%


\author{Bobby Chen}
\authornote{Authors listed alphabetically.}
\authornote{Formerly at Uber Technologies, Inc.}

\author{Siyu Chen\footnotemark[1]}
\author{Jason Dowlatabadi\footnotemark[1]}
\author{Yu Xuan Hong\footnotemark[1]}
\author{Vinayak Iyer\footnotemark[1]}
\author{Uday Mantripragada\footnotemark[1]}
\affiliation{
  \institution{}
  \city{}
  \state{}
  \country{}
}

\author{Rishabh Narang\footnotemark[1]}
\author{Apoorv Pandey\footnotemark[1]}
\author{Zijun Qin\footnotemark[1]}
\author{Abrar Sheikh\footnotemark[1]}
\author{Hongtao Sun\footnotemark[1]\footnotemark[2]}
\author{Jiaqi Sun\footnotemark[1]}
\affiliation{
  \institution{Uber Technologies, Inc}
  \city{San Francisco}
  \state{CA}
  \country{USA}
}
\email{mhwalker,zijunq@uber.com}

\author{Matthew Walker\footnotemark[1]}
\author{Kaichen Wei\footnotemark[1]}
\author{Chen Xu\footnotemark[1]}
\author{Jingnan Yang\footnotemark[1]}
\author{Allen T. Zhang\footnotemark[1]}
\author{Guoqing Zhang\footnotemark[1]}
\affiliation{
  \institution{}
  \city{}
  \state{}
  \country{}
}


\begin{abstract}
    Budget allocation of marketplace levers, such as incentives for drivers to complete certain
    trips or promotions for riders to take more trips have long been both a technical and business challenge
    at Uber.
    It is crucial to understand the impact of lever budget changes on the market and to estimate their cost
    efficiency given the need to achieve predefined budgets, where the eventual goal is to find the optimal allocations
    under those constraints that maximize some objective of value to the business.
    In this paper, we introduce an end-to-end machine learning and optimization procedure to automate
    budget decision-making for cities where Uber operates.
    This procedure relies on a suite of applications, including feature store, model training and serving, optimizers and
    backtesting to measure the prediction and causal accuracy.
    We propose a state-of-the-art deep learning (DL) estimator based on S-Learner that leverages massive amount of user
    experimental and temporal-spatial observational data.
    We also built a novel tensor B-Spline regression model to enforce efficiency shape control while retaining the
    sophistication of the DL models’ response surface, and solved the high-dimensional optimization problem
    with Alternating Direction Method of Multipliers (ADMM) and primal-dual interior point convex optimization.
    This procedure has demonstrated substantial improvement in Uber's ability to allocate resources efficiently.
\end{abstract}




\maketitle

\section{Introduction}\label{sec:introduction}

Each year, Uber manages a multi-billion dollar marketplace.
In the third quarter of 2023 alone~\cite{uber_2023_q3_earnings}, Uber's mobility business recorded \$17.9 billion in
Gross Bookings (GBs)\@. To attract riders and drivers to use the products and to influence the behavior of the
marketplace,  Uber invests in various marketplace levers, allocating budgets across different regions and lever types to
optimize business objectives.
Historically, budget allocation was done manually, with teams reviewing data and adjusting numbers in spreadsheets based
on personal experience and past allocations. This method did not effectively utilize historical data on rider and
driver behavior under different conditions, lacked a clear objective, and was labor-intensive, making it difficult to
measure its effectiveness quantitatively.
We present an automated system that leverages causal deep learning models to predict how each marketplace
lever affects driver and rider behavior and a set of optimization algorithms to decide optimal budget allocation across
levers and regions.
The causal deep learning model leverages historical data to predict supply, demand, and a custom marketplace objective in the upcoming
week for each region and a given set of incentives.
The system then uses a smoothing model to generate high level functions between incentive budgets and the outcome for the
optimization stage.
Finally, the system finds the optimal budget allocation for each region and incentive lever subject to business
operations constraints, using an optimization system to solve a non-convex problem.
In addition, we also devised a framework and system to measure the business impact of the budget allocation when
compared to a baseline allocation.
This system is deployed in production on a weekly basis to allocate incentive budgets, and is proven to generate
significant impact to Uber's business.
This work contains the following contributions:
\begin{itemize}
	\item Introducing a novel S-Learner architecture that combines temporal-spatial observational data with finer-granularity experimental data in its learning objective.
	\item Developing a student model technique using b-splines that enforces efficiency, business intuition, and teacher model response shape control while marginalizing non-control features.
	\item Creating a fully distributed optimizer using ADMM and primal-dual interior point convex optimization, implemented on Ray to handle complex, nonlinear, non-convex problems.
	\item Proposing a new method for evaluating counterfactual outputs without a clean experimental environment for standard approaches like Randomized Control Trials (RCTs) or switchbacks.
\end{itemize}


\section{Background and Related Work}\label{sec:background-and-related-work}

Uber, being a publicly traded company, gives guidance to the public about expectations for performance on certain business metrics for each quarter.
Two important metrics are Gross Bookings (GB), which is a measure of how much business was conducted on the Uber platform, and Net Income (NI), which is the generally accepted accounting practice definition of profit~\cite{uber_2023_q3_earnings}.
The internal teams at Uber must make prediction on how different products can affect market and therefore affect the key financial metrics.
This financial task can be formulated as follows:
\begin{equation}
\label{eqn:multiweek_problem}
\begin{array}{ll@{}ll}
\text{maximize} & \sum\limits_{w \in \textrm{weeks}, c \in \textrm{cities}} & obj_{w,c}(\mathbf{b_{w,c}})&\\
\text{subject to} &  \sum\limits_{w \in \textrm{weeks}, c \in \textrm{cities}, l \in \textrm{levers}} & b_{w,c,l} = B & \\
& & b_{w,c,l} \geq h_{w,c,l} & \\
& & b_{w,c,l} \leq g_{w,c,l} &
\end{array}
\end{equation}
where weeks are calendar weeks, cities are approximately independent geographical entities, and levers are different
marketplace programs that Uber can allocate money to in order to affect marketplace outcomes.
The function $obj$ is some objective of importance to the business.
The variables $b_{w,c,l}$ represent budgets for specific weeks, cities, and levers, which can be positive or negative; we use the convention that positive budgets correspond to Uber spending. The variable $B$ is a total budget constraint on the system. The variables $h_{w,c,l}$ and $g_{w,c,l}$ represent floors and ceilings, respectively, on how much money can be spent in a lever in a specific city-week, which can again be positive or negative.

The objective of Equation~\ref{eqn:multiweek_problem} is to maximize business through-put subject to a fixed budget constraint.
The challenge is that the effects of budgeting a certain lever may not be linear and there may be substantial interaction effects between levers.
In the following discussion, we confine ourselves to discussing only Uber's Rides business and further confine the problem to a single week and the budgets are full-week treatments.

\subsection{Causal Marketplace Forecasting}

A critical component of solving this problem is how well we can predict the impact of changing the budgets under our control on $obj$\@.
Imagine a world that is described by an oracle:
\begin{equation}
\label{eqn:oracle}
obj(\mathbf{b}, \mathbf{X})
\end{equation}
where $\mathbf{b}$ and  $\mathbf{X}$ are the feature vectors for anything under our control and not under our control.

Our task is to produce an estimator of that oracle for a future week:
\begin{equation}
\label{eqn:estimator}
\widehat{obj}(\widehat{\mathbf{b}}, \widehat{\mathbf{X}}),
\end{equation}
where $\widehat{\mathbf{b}}$ is what we believe the feature vector under our control will be valued and $\widehat{\mathbf{X}}$ is what we predict any feature of interest will be valued. Importantly, $dim(\mathbf{b}) \stackrel{?}{=} dim(\widehat{\mathbf{b}})$ and $dim(\mathbf{X}) \stackrel{?}{=} dim(\widehat{\mathbf{X}})$, which is to say, neither control variables nor other confounding features can be guaranteed to be observed. The control variables have continuous values, in this case budgets.

Let us consider how data can be collected in this context. We are concerned with \emph{marketplace} outcomes, which are outcomes aggregated over some finer granularity. In our case, we are looking at outcomes aggregated per city per week. The finer granularity, in the case of Uber, could be riders, drivers, trips, etc. In a marketplace, it is generally assumed that there can be interference effects between treatment groups during experiments\cite{hu2022average}. Therefore, experimental results obtained by conducting a randomized controlled trial on a higher granularity treatment entity may not extrapolate to applying a treatment to the entire population of treated entities.

Since the standard assumptions of collecting unconfounded, independent, and identically distributed data do not apply, we can not rely on standard techniques or bounds to guarantee performance of any estimator. Instead, we must develop empirical techniques against which to test any estimator using both predictive accuracy, referring to how well absolute outcomes of the estimator agree with test data, as well as causal accuracy, referring to how well derivatives of the estimator agree with test data.

\subsection{Related Work}
\textbf{Time series forecasting} continues to be an area of significant interest in the literature. Autoregressive
statistical approaches related to ARIMA \cite{boxARIMA} remain performant in certain domains, including stock
prediction \cite{ariyo2014stock} and pandemic spread \cite{benvenuto2020application} prediction. In light of the success
of transformers \cite{VaswaniSPUJGKP17} in the natural-language processing space, many proposals have been made to apply transformers to time series prediction problems in order to model both multi-variate dependencies and overall complexity in the time series \cite{li2019enhancing, lim2021temporal, zhou2021informer, zhou2022film, zhou2022fedformer}. Recent results showing the effectiveness of simple linear models on standard datasets \cite{zeng2023transformers} has led to additional recent work on MLP-only approaches \cite{chen2023tsmixer}. In contrast to the literature mentioned above, which focuses primarily on high predictive accuracy, as measured generally by mean-squared error (MSE) and mean absolute error (MAE), we are concerned also with causal accuracy of the model. Furthermore, identification of a generally superior model architecture has so far defied researchers in the time series forecasting space, so performant models in new domains remain of interest.

\textbf{Causal ML} approaches have explored how to infer heterogeneous treatment effects using multiple different strategies\cite{imai2013estimating, athey2016recursive, hill2011bayesian, hahn2020bayesian, powers2018some, shalit2017estimating, du2019improve} and showing that, under certain conditions, certain error terms can be substantially reduced\cite{chernozhukov2018double, kunzel2019metalearners, nie2021quasi}. This body of work nearly entirely relies on two assumptions that we cannot rely on in practical marketplace applications. Relying on the notation that $X_i$ are per-unit features, $Y_i \in \mathbb{R}$ is the observed treatment effect and $W_i \in {0, 1}$ is the binary treatment assignment, the first assumption is that the treatment assignment $W_i$ is unconfounded, i.e. $\{Y_i(0), Y_i(1)\} \perp W_i | X_i$~\cite{nie2021quasi}. This is nearly always false in any available observational marketplace data. To give an obvious example: larger cities have larger budgets. The second assumption is that any data of interest is independent and identically distributed. Neither of these conditions hold under realistic marketplace conditions. An example of the first is that the level of treatment in one week may impact how users return in the following week, resulting in different returns. An example of the second is that there could be a major sporting event one week that results in very different user behavior compared to any other week. While some of the techniques are informative, the theoretical findings have no applicability to the practical case.


\section{Methodology}\label{sec:methodology}
In this section, we will focus on details of key components of the budget allocation system: the user causal effect estimator,
smoothing layer with B-Spline functions, optimizer and the business value evaluation framework.
A diagrammatic view of the entire system can be found in Appendix~\ref{sec:appendix_system}.


\subsection{User Causal Effect Estimator}\label{subsec:estimators}
To better leverage A/B experiments data which contains the information of user
behaviors between treatment and control groups, We adopted the S-metalearner~\cite{kunzel2019metalearners}
framework with a ResNet\cite{he2015deep} model to estimate individual treatment effects (ITE).

We have explored other metalearners such as the stacked R-metalearner, and the
decision towards S-metalearner took various considerations into account including feasibility
of available computing resources, the implementation and maintenance efforts required from a practical standpoint,
and experiment results in Appendix~\ref{sec:appendix_estimator} show S-metalearner provided better incremental efficiency with diminishing returns, aligning
well with business intuition.

As illustrated in Fig.~\ref{fig:base_learner}, the model consists of two stages:

\textbf{Embedding stage}
\begin{itemize}
    \item \emph{Sparse embeddings}. Our input features consist of some categorical sparse features.
    Since DNN’s are usually good at handling dense numerical features, we map the sparse features into
    trainable dense vectors using embeddings.
    \item \emph{Dense feature extractor}. We also extract an embedding representation for dense numerical features
    in order to better capture higher order interactions between the dense features.
    The dense feature embeddings are generated through multilayer perceptrons\cite{haykin1994neural} (MLP).
\end{itemize}

\textbf{Residual network stage}. Sparse and dense feature embeddings are concatenated and
fed into a series of ResNet blocks and eventually generate predictions.
ResNet blocks are chosen instead of MLP in order to better handle vanishing gradients and reduce the risk
of overfitting when feature sizes are large.

During analysis of deployed budget recommendations, we also diagnosed and resolved an
endogeneity~\cite{doi:10.1177/0149206320960533} issue that is naturally introduced from clustering data across cities
and weeks during model training.
It led to possible cross week extrapolations as the training data from different weeks could be endogenous and more
severely, the treatment-control effects are masked out by cross-week data.
This issue was resolved by adding additional positional encoding for the week in the training and serving data.

\begin{figure}[htbp]
\centering
\includegraphics[width=0.37\textwidth]{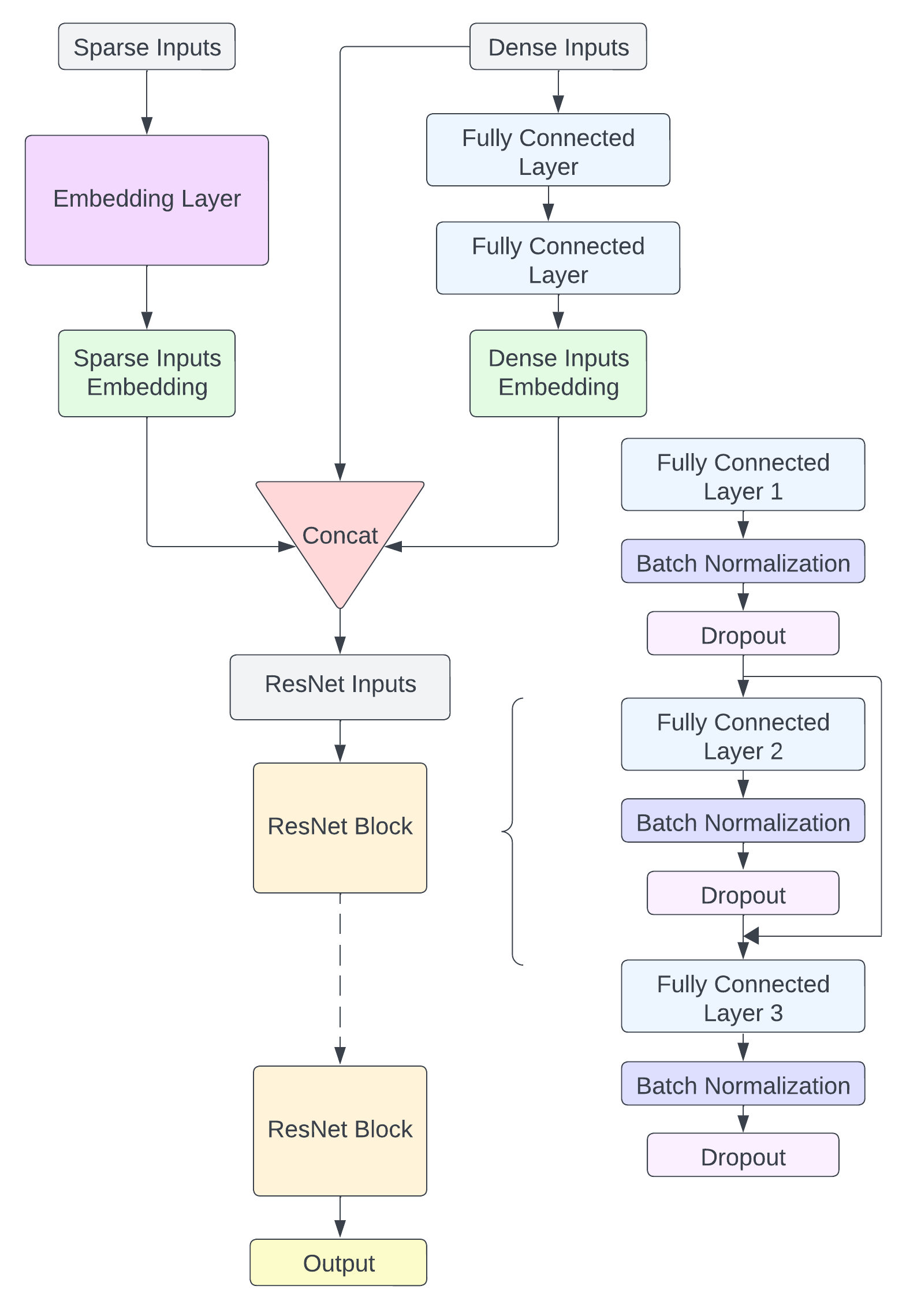}
\caption{Driver model base learner architecture}
\label{fig:base_learner}
\end{figure}


\subsection{Smoothing Layer}\label{subsec:smoothing-layer}
The DL model in Section~\ref{subsec:estimators} produces a mapping from budgets to marketplace outcomes,
which can be fed into an optimizer.
However, most optimization routines require numerous passes over the surface,
necessitating many evaluations of the DL model, which is computationally costly, potentially prohibitively so.
Additionally, the mapping produced by DL models does not necessarily satisfy business
intuition (e.g., non-negativity, monotonicity, convexity).
Enforcing such business intuition using flexible functional forms is generally intractable.
For instance, verifying polynomial non-negativity~\cite{blekherman2012semidefinite} and convexity~\cite{ahmadi2013np}
is NP-hard.

Our goal is to generate a differential surface for derivative-based optimization algorithms but also
to ensure low-cost evaluations and compliance with business intuitions.

We solve this bottleneck by capturing the DL surface with a lower-dimensional representation.
We begin by leveraging the Adaptive Sparse Grids (ASG) algorithm~\cite{smolyak1963quadrature, zenger1991sparse, griebel1991parallelizable, griebel1992combination, bungartz2004sparse}
to translate the DL surface to a computationally cheap grid-based representation.
Typically, grid-based representations suffer from a curse of dimensionality as the number of grid nodes increases
exponentially with dimensionality (in this case, the number of marketplace levers we are optimizing).
ASG solves this issue by limiting grid nodes to locations that result in the highest contributions
to accuracy and it serves as the foundation of the smoothing layer.

We then select B-spline functions as our smoothing model due to their flexibility and analytical derivative
forms~\cite{eilers1996flexible}.
Unlike monotonic regression\cite{10.1145/3357384.3357835}, B-splines offer the advantage of fitting derivatives,
thereby avoiding optimization issues and aligning better with business intuition by providing continuous
lever efficiencies rather than step-wise constants, B splines with quadratic basis functions allow us to enforce monotonicity and convexity,
enhancing tractability~\cite{pya2015shape}, while Cox regression\cite{Yu_Wu_Cai_Liu_Zhang_Gu_Zeng_Gu_2021}, often used in survival analysis,
yields monotonically decreasing curves and is non-convex.
Additionally, B-spline functions can adopt shape controls such as spot and universal derivative bounding,
in a multidimensional treatment scenario, without parametric assumptions,
thus reducing the risk of model misspecification.
Furthermore, we enhance the model causal accuracy by integrating network effects-corrected
budget efficiency experimental data through an additional penalty term in the objective function.

\begin{algorithm}[htbp]
    \caption{Adaptation algorithm: Build the grid iteratively}
    \SetKwBlock{Repeat}{repeat}{}
    \SetKwBlock{Until}{until}{}
    \label{alg:asg}
    Begin with a starting grid\;
    \Repeat{Compute hierarchical surplus at nodes; Find nodes with high hierarchical surplus, Add children nodes around these nodes}
    \textbf{Until} Recall hierarchical surplus measures residual error at each node, stop when all children nodes
    below maximum acceptable residual error $|\epsilon_{add}|$\;
\end{algorithm}

\subsubsection{Adaptive Sparse Grid}\label{subsubsec:adaptive-sparse-grid}
Adaptive sparse grids(ASG), constructed by preserving only a subset of the nodes on a dense grid, have been shown to
maintain a comparable level of error while alleviating this curse of dimensionality \cite{brumm2017adaptive, schaab2022adaptive}.
When the function representation is derived from linearly interpolating between nodes, errors in the representation stem
from not having sufficient resolution in areas where the function is highly nonlinear.
ASG iteratively refine the grid and add nodes where errors remain high.
The adaptation process is governed by a score called the hierarchical surplus, which essentially measures how much the
underlying objective function at a node deviates from a linear interpolation via the surrounding nodes.
High values of the surplus imply uncaptured local curvature, and suggest additional nodes nearby could be useful.
As an example, Figure \ref{fig:asg} below shows how the adapted sparse grid places greater density around the nonlinear portion of
the function $f(x, y) = (|0.5 - x4 - y4| + 0.1) - 1$.

\begin{figure}[htbp]
\centering
\includegraphics[width=0.4\textwidth]{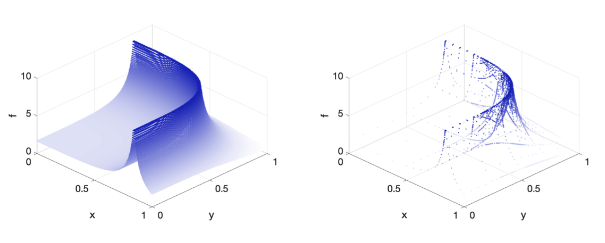}
\caption{Dense vs. Sparse Grid Representation}
\label{fig:asg}
\end{figure}

\subsubsection{B-Spline with Business Penalty}\label{subsubsec:b-splie-smoothing}
We formulate each city's cost surface smoothing as a constrained least square problem\cite{boyd2018introduction} in Equation~\eqref{eq:bsplineobj}.
The fitting objective function consists of two parts.
The first part is to minimize the sum of squared difference between B-spline smoothing model inferred
$\hat{obj}_{s}(\boldsymbol{x^g})$ and DL model inferred $obj_{r}(\boldsymbol{x^g})$ over all budget points
$\boldsymbol{x^g}=\{x_{l}|l=1,\dots,N \}$ in search space grid $ G$; while the second part is to minimize the sum of squared difference between B-spline inferred budget
efficiency and A/B experiment measured budget efficiency, which is essentially the $IOB$ introduced in Section \ref{subsec:evaluation}.
The variance of experimental data $\sigma^2_l(\boldsymbol{x^m})$ is incorporated as inverse weights to handle heteroscedasticity \cite{strutz2011data}.
The hyperparameter $\lambda$ determines the tradeoff between the objectives which can be tuned via
cross-validation to optimize business impact defined in Equation~\ref{eq:business_impact}.

\begin{align}
   & \underset{\hat{obj}_{s}(\cdot)}{\text{minimize}} &&  \sum\limits_{\boldsymbol{x^g} \in G} (\hat{obj}_{s}(\boldsymbol{x^g}) - obj_{r}(\boldsymbol{x^g}
    ))^2  \label{eq:bsplineobj} \\
    & & &  +  \lambda \sum\limits_{\boldsymbol{x^m} \in M} \sum_{l=1}^{N} \frac{1}{\sigma^2_l(\boldsymbol{x^m})} (\hat{IOB}_l(\boldsymbol{x^m}) -
    IOB_l(\boldsymbol{x^m}))^2  \notag \\
    & \textrm{s.t.} &&   \hat{obj}_{s}(\boldsymbol{x}) = \sum_{i_1=0}^{n_1-1} \cdots \sum_{i_N=0}^{n_N-1} a_{i_1\dots i_N} \prod_{l=1}^{N} B_{i_l}^{r_l}(x_l;\boldsymbol{t_l}) \\
    &  & &  \hat{IOB}_{l} (\boldsymbol{x}) =\frac{\partial \hat{obj}_s(\boldsymbol{x})}{\partial x_l} \geq 0, \ \forall l\in \{1,\dots,N\} \label{bsplinenonnegativity} \\
    & & &  \frac{\partial \hat{IOB}_{l}(\boldsymbol{x})}{\partial x_l} \leq 0 , \ \forall l\in \{1,\dots,N\} \label{bsplineconcavity}
\end{align}

where $ B_{i_l}^{r_l}(x_l;\boldsymbol{t_l})$ is the $i_l$th B-Spline basis function of degree $r_l$ over the knots $\boldsymbol{t_l}$ for the $l$th incentive lever.
The number of basis functions for $l$th lever is denoted as $n_l$. There are $\prod_{l=1}^{N}n_l$ B-spline basis function coefficient $a_{i_1\dots i_N}$ to fit.
Ignore the lever index $l$ hereafter for simplicity.
Knots $\boldsymbol{t}$ consist of $n - r + 1$ internal knots plus $r$ boundary knots on each side.
The complete knots vector $\boldsymbol{t}$ is equal to $\{t_{0},\cdots, t_{n+r}\}$.
Each basis function is defined by a recursive relationship.
The $0$th degree B-spline has the form
\begin{align} \label{eq:0thbasis}
    B_{i}^{0}(x;\boldsymbol{t}) = 1, \ \textrm{if} \ t_{i} \leq x < t_{i+1}, \ \textrm{otherwise} \ 0
\end{align}
and higher degree B-splines are constructed as
\begin{align}\label{eq:higherbasis}
    & B_{i}^{r}(x;\boldsymbol{t}) = \frac{x - t_{i}}{t_{i+r} - t_{i}}B_{i}^{r-1}(x;\boldsymbol{t}) + \frac{t_{i+r+1}-x}{t_{i+r+1} - t_{i+1}}B_{i+1}^{r-1}(x;\boldsymbol{t})
\end{align}

The sufficient conditions of the non-negative return of $l$th lever over $[t_{l,r_l},t_{l,n_r}]$ can be shown as
\begin{align}\label{eq:1stsufficient}
    a_{i_1\dots i_l\dots i_N} - a_{i_1\dots i_l-1\dots i_N} \geq 0, \ \forall i_l
\end{align}

and the sufficient conditions of the diminishing marginal return of $l$th lever can be shown as
\begin{align}\label{eq:2ndsufficient}
   &  \frac{a_{i_1\dots i_l\dots i_N}-a_{i_1\dots i_l-1\dots i_N}}{\Delta_{i_l}}   \leq \frac{a_{i_1\dots i_l-1\dots i_N}-a_{i_1\dots i_l-2\dots i_N}}{\Delta_{i_l-1}}, \forall i_l
\end{align}
where $\Delta_{i_l} = t_{l,i_l+r_l}-t_{l,i_l}$.

For a univariate B-spline with degree $r=2$, the above conditions are also necessary. See Appendix \ref{sec:appendix} for details.
\Cref{eq:1stsufficient,eq:2ndsufficient} are used to replace \cref{bsplinenonnegativity,bsplineconcavity} to make the optimization problem tractable.

%


\subsection{Optimizer}\label{subsec:optimization}

We use an optimization system to solve a global problem as described in Equation~\ref{eqn:multiweek_problem}, which seeks to solve an objective maximization problem under constraints. While, as discussed in Sec.~\ref{subsec:smoothing-layer}, we are able to guarantee monotonicity and concavity in individual dimensions, that does not extend to the entire surface.

In addition to the predictions from the $obj$ estimator, we also introduce a penalty term in the
form of Hellinger distance~\cite{hellinger1909neue} that forces the system to prefer allocations similar to a reference allocation:
\begin{equation}
\label{eqn:penalty_term}
\begin{array}{ll@{}ll}
pen(\mathbf{b}_w, \mathbf{b}_0)= &-\sqrt[-3]{\textrm{max}(\frac{||\mathbf{b}_w||}{||\mathbf{b}_0||}, \frac{||\mathbf{b}_0||}{||\mathbf{b}_w||}}) \times \\
& \sum\limits_{c \in \textrm{cities}} \alpha_c \sum\limits_{l \in \textrm{levers}} \kappa_l |\sqrt{\frac{|b_{w,c,l}|}{||\mathbf{b}_w||}} - \sqrt{\frac{|b_{0,c,l}|}{||\mathbf{b}_0||}}|^2,
\end{array}
\end{equation}
where $\mathbf{b}_w$ is the budget vector for the week to be optimized, $\mathbf{b}_0$ is the reference budget vector, $\alpha_c$ and $\kappa_l$ are scaling terms per city and lever, respectively, and the cube root term reduces the penalty term size as the budget difference increases, with the logic that the reference allocation should not be preferred if the total difference between the allocations is large.

In order to solve the non-linear, non-convex problem, which we know remains differentiable, we adopt an Alternating
Direction Method of Multipliers (ADMM)~\cite{wang2019global} approach. Specifically, we separated the single-city
cross-lever problem and the cross-city problem, and solved the problems iteratively.

Here is the reformation of the problem:
\begin{equation}
\label{eqn:admm_formulation}
\begin{array}{ll@{}ll}
    & \text{Minimize} \sum\limits_{c \in \textrm{cities}} f_c(\mathbf{b}_c) + g(\mathbf{z}) \\
    & \text{subject to}~ \mathbf{b}_c - \mathbf{z}_c = 0, c \in \textrm{cities}
\end{array}
\end{equation}

The resulting ADMM algorithm is,
\begin{equation}
\label{eqn:admm_updates}
\begin{array}{ll@{}ll}
    & \mathbf{b}_c^{k+1} := \mathop{argmin}\limits_{\mathbf{b}_c}(f_c(\mathbf{b}_c) + \rho / 2 \parallel \mathbf{b}_c - \mathbf{z}^k + \mathbf{y}_c^{k} \parallel_2^2) \\
    & \mathbf{z}^{k+1} := \mathop{argmin}\limits_{\mathbf{z}}(g(\mathbf{z}) + \rho / 2 \sum\limits_{c \in \textrm{cities}}  \parallel \mathbf{b}_c^{k+1} - \mathbf{z}_c + \mathbf{y}_c^{k}  \parallel_2^2) \\
    & \mathbf{y}_c^{k+1} := \mathbf{y}_c^{k} + \mathbf{b}_c^{k+1} - \mathbf{z}_c^{k+1},
\end{array}
\end{equation}
where we slightly abuse the notation, as $\mathbf{z}$ is the concatenation of $\mathbf{z}_c \forall c \in \textrm{cities}$.
From the equations, we can summarize the algorithm in 3 steps:
\begin{itemize}
    \item Update $\mathbf{b}$ step, this step considers all the single city constraints.
    \item Update $\mathbf{z}$ step, this step brings all the city together and only considers cross city optimization.
    \item Update $\mathbf{y}$ step, this step updates the center of the update x step.
\end{itemize}
Because our problem is close to convex, and the constraints applied to our smoothing model, by tuning $\rho$,
we are able to ensure that the overall problem has a positive, semidefinite Hessian.
The implementation of the algorithm is detailed in Appendix~\ref{sec:appendix_optimizer}.


\subsection{Business Value Evaluation}\label{subsec:evaluation}
In contrast to many other machine learning applications, which are mainly evaluated on predictive accuracy, our business use
case requires us to predict marketplace outcomes for the coming weeks and generate budget allocations that maximize
those outcomes.
Therefore, we need a customized evaluation framework to measure the quality of these budget allocations.
One common approach, running large-scale experiments in the marketplace, is prohibitively costly and impractical for continuous model evaluation.
Instead, we rely on existing A/B experiments that adjust budgets at the user level to estimate our business impact.

The Business Value Evaluation (BVE) framework combines existing A/B experimental data with simulated
allocations to produce estimates that can be directly interpreted as weekly business impact. In particular, this can be calculated as:
\begin{equation}
    \label{eq:business_impact}
    \text{Business Impact} = \Sigma_{l,c} \int_{B_{l,c,old}}^{B_{l,c,new}} IOB_{l,c} dB
\end{equation}
Where $l$ indexes incentive levers and c denotes cities, and IOB is the implied return on additional
budget from the existing A/B experimental data. To compute the integral, we take two approaches.
(1) A linear approximation of the integral which implies simply computing the difference in budget between the optimal and
actual allocations, then multiplying these differences by the return on this incremental budget.
(2) A log-linear relationship between budgets and IOB and estimates this relationship from historical data. This ensures that we
capture the curvature in the relationship between budget and IOB.

The key input of BVE is the incremental return on budget (IOB) which is estimated from user level A/B experiments which either make treatment more generous or target additional units.
Given these continuously running experiments, we can measure the lift of an additional dollar budget on our $obj$ with the
following regression as:
\begin{align*}
obj_i = \alpha + \beta \times Treatment_i + X_i \theta + \epsilon_i
\end{align*}
Where \(i\) refers to a treated unit and \(X_i\) are a vector of pre-XP covariates.
Adding in controls into this regression framework allows us to soak up idiosyncratic noise and increase precision \cite{athey2016econometrics}.
The coefficient  gives us the incremental lift on $obj$ per user (iO). Moreover, for each of these XPs, the incremental budget per user (iB) is deterministic and given by the XP design.
This implies that we can compute  \(\frac{\Delta obj}{\Delta B}\).

The main challenge in computing IOB based on unit level A/Bs at Uber is the presence of network effects or spillover effects.
What makes this a unique problem without standard off the shelf solutions is that the Uber marketplace has multiple rationing mechanisms during periods of relative undersupply; namely prices and ETAs (expected time to arrival).
This implies that the Stable Unit Treatment Value Assumption (SUTVA) is often violated, where outcomes for one unit depends on the treatment assignments of others.
As a simple example, allocating budget towards incentives which boost demand when there is acute undersupply can lead to user level A/B effects
grossly overestimating the true marketplace effect. The key mechanism behind this is that treated users on the demand side of the market might cannibalize supply which would have been otherwise available to the control group users.
To account for this issue, we developed an economic model which predicts the incremental marketplace
changes in endogenous quantities such as our $obj$ to changes in exogenous quantities like supply and demand. This model is characterized by a set of key elasticities which govern
how both users and aggregate marketplace outcomes respond to changes in Uber's rationing mechanisms; prices and ETAs. The paper closest in spirit to our approach is \cite{castillo2017surge} which develops a marketplace model to understand
how pricing can prevent the market from reaching undesirable equilibria.
Our approach is also similar to recent papers in the literature \cite{bimpikis2019spatial, castillo2023benefits, al2021drives}
which develop marketplace models in ride-sharing markets to study optimal pricing and matching policies. With key elasticities estimated from previously run XPs,
we can predict the change in our $obj$ due to changes in model inputs such sessions and supply hours.
We then use a first-order Taylor approximation to translate A/B outcomes into marketplace outcomes in the following way:
\begin{align*}
    \text{Marketplace Effect on Obj} = (\text{Effect of Model Inputs on Obj}) \\
    \times ~ (\text{A/B Effect on Model Inputs})
\end{align*}

Taking driver promotions as example, this implies that we use the calibrated model’s implied effect on marketplace outcomes
such as our $obj$ due to changes in supply and the A/B effect on supply to compute the marketplace effect on $obj$\@.

\section{Evaluation and Results}\label{sec:evaluation-and-results}
We conducted comprehensive backtesting using Uber marketplace historical data to evaluate the efficacy and robustness of our proposed
methodology. This approach allows us to measure incremental out-of-sample performance and the business impact of
counterfactual budget allocations from our system, all while fully accounting for complexities in the underlying data
generating process

We begin with predictive accuracy to ensure our model captures marketplace dynamics effectively.
We compute wMAPE (weighted Mean Absolute Percentage Error) and wBIAS (weighted BIAS) at the city-week level,
appropriately weighting cities of varying sizes and matches the granularity of the budget allocation decision.
Table \ref{tab:prediction_accuracy} shows how these metrics vary across methodologies and regions.
Across both regions, the Causal DL model substantially outperforms the stacked Estimator model in terms of wMAPE and wBIAS\@.

We complement these predictive accuracy metrics with business impact metrics from the aforementioned
Business Value Evaluation framework in Section \ref{subsec:evaluation}. Internally, we compute business impact from Equation \ref{eq:business_impact} in terms
of $obj$ when comparing candidate models for production usage. While we are unable to share this directly, we can share
the following closely related metric for each model:
\begin{align*}
    \text{Marginal Efficiency} = \frac{ \Sigma_{l,c} \text{B}_{l,c,opt}  \times ~ \text{IOB}_{l,c}}{\Sigma_{l,c} \text{B}_{l,c,opt}}
\end{align*}

Marginal efficiency captures the value of an additional dollar of budget allocated proportionally to all cities and levers by weighting
efficiency estimates accordingly. High marginal efficiency implies budgets are being deployed in an efficient manner.
Table \ref{tab:business_value_evaluation} reports the percentage change in this metric against a baseline model for our backtests which show that the Causal DL model with model enhancements outperforms the baseline model in both regions.

\begin{table}[htbp]
  \centering
  \caption{Prediction accuracy between stacked Estimator and Causal ML model}
  \label{tab:prediction_accuracy}
  \small
\begin{tabularx}{\linewidth}{XX|X|X}
    \toprule
    \multicolumn{2}{c}{\textbf{Dataset}} \vline & \textbf{Region US} & \textbf{Region BR} \\
    \hline
    \midrule
    Method & Metric & \\

    \multirow{2}{*}{\shortstack[1]{LGBM baseline}} & wMAPE \% & 7.109  & 7.454 \\
                                                         & wBIAS \% & -0.687 & -1.863 \\
    \hline
    \multirow{2}{*}{\shortstack[1]{Causal DL}} & wMAPE \% & \textbf{3.264} & \textbf{5.622} \\
                                               & wBIAS \% & \textbf{-0.115} & \textbf{0.645} \\
    \bottomrule
  \end{tabularx}
\end{table}

\begin{table}[htbp]
  \centering
  \caption{Marginal efficiency by method and region, percentage improvements relative to baseline model}
  \label{tab:business_value_evaluation}
  \begin{tabular}{c|c|c}
    \hline
    \textbf{Method} & \textbf{Region US} & \textbf{Region BR} \\
    \hline
    Causal DL w/ business penalty & 9.23\% & 2.59\% \\
    Causal DL w/ endogeneity fix & 8.00\% & 0.82\% \\
    \hline
  \end{tabular}
\end{table}

Finally, Figure \ref{fig:city_week_heatmap} shows an example of how budget allocations differ across select cities (y-axis)
and weeks (x-axis) between two different models. Colors and annotations indicate the percentage change in budgets across the models.
Importantly, note how the budget differences are mostly small. This allows us to comfortably
use a linearized version of Equation \ref{eq:business_impact} to measure business impact.

\begin{figure}[htbp]
\centering
\includegraphics[width=0.75\columnwidth]{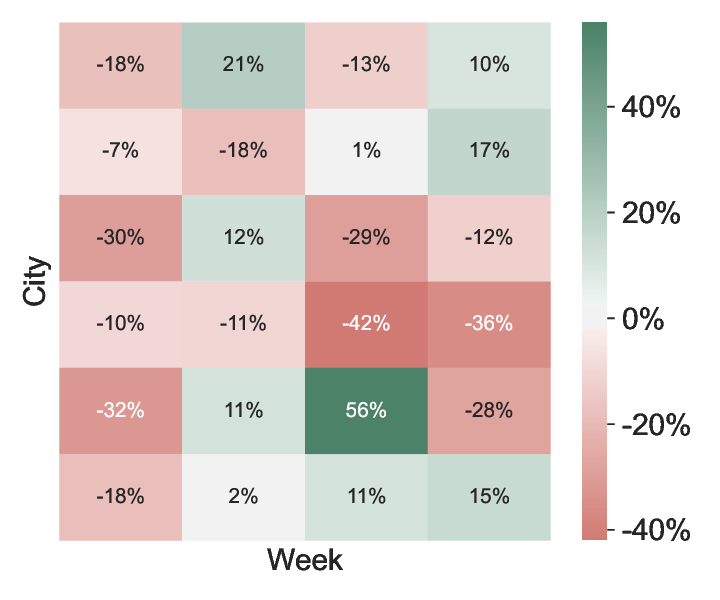}
\caption{Budget differences between two models}
\label{fig:city_week_heatmap}
\end{figure}

\section{Conclusion}\label{sec:conclusion}

We present an end-to-end causal machine learning and optimization system designed to enhance Uber's marketplace budget allocation process.
This system automates weekly budgeting decisions to maximize marketplace objectives within predefined constraints, improving overall operational efficiency.
Our approach addresses a high-dimensional optimization problem involving various marketplace levers and their respective budgets.
To solve this, we employ a deep learning estimator based on an S-Learner approach, leveraging extensive experimental and temporal-spatial observational data to estimate causal effects.
We also incorporate a tensor B-Spline regression model that captures the detailed response surface of DL models while ensuring practical efficiency.
Additionally, our system utilizes ADMM and primal-dual interior point convex optimization techniques, implemented on Ray, to handle large-scale nonlinear, non-convex problems efficiently.
The deployment of this system not only simplifies decision-making but ensures these decisions are data-driven and aligned with Uber's strategic objectives.

\begin{acks}
This work was supported by Uber Technologies Inc.
\end{acks}

\bibliography{./main-workshop.bib}

\begin{thebibliography}{10}

\bibitem{uber_2023_q3_earnings}
Uber announces results for third quarter 2023, 11 2023.

\bibitem{ahmadi2013np}
Amir~Ali Ahmadi, Alex Olshevsky, Pablo~A Parrilo, and John~N Tsitsiklis.
\newblock Np-hardness of deciding convexity of quartic polynomials and related
  problems.
\newblock {\em Mathematical Programming}, 137:453--476, 2013.

\bibitem{al2021drives}
Motaz Al-Chanati and Vinayak Iyer.
\newblock What drives the efficiency in ridesharing markets? evidence from
  austin, texas.
\newblock {\em Available at SSRN 3959363}, 2021.

\bibitem{andersen2020cvxopt}
Martin Andersen, Joachim Dahl, and Lieven Vandenberghe.
\newblock Cvxopt: Convex optimization.
\newblock {\em Astrophysics Source Code Library}, pages ascl--2008, 2020.

\bibitem{ariyo2014stock}
Adebiyi~A Ariyo, Adewumi~O Adewumi, and Charles~K Ayo.
\newblock Stock price prediction using the arima model.
\newblock In {\em 2014 UKSim-AMSS 16th international conference on computer
  modelling and simulation}, pages 106--112. IEEE, 2014.

\bibitem{athey2016econometrics}
Susan Athey and Guido Imbens.
\newblock The econometrics of randomized experiments, 2016.

\bibitem{athey2016recursive}
Susan Athey and Guido Imbens.
\newblock Recursive partitioning for heterogeneous causal effects.
\newblock {\em Proceedings of the National Academy of Sciences},
  113(27):7353--7360, 2016.

\bibitem{benvenuto2020application}
Domenico Benvenuto, Marta Giovanetti, Lazzaro Vassallo, Silvia Angeletti, and
  Massimo Ciccozzi.
\newblock Application of the arima model on the covid-2019 epidemic dataset.
\newblock {\em Data in brief}, 29:105340, 2020.

\bibitem{bimpikis2019spatial}
Kostas Bimpikis, Ozan Candogan, and Daniela Saban.
\newblock Spatial pricing in ride-sharing networks.
\newblock {\em Operations Research}, 67(3):744--769, 2019.

\bibitem{blekherman2012semidefinite}
Grigoriy Blekherman, Pablo~A Parrilo, and Rekha~R Thomas.
\newblock {\em Semidefinite optimization and convex algebraic geometry}.
\newblock SIAM, 2012.

\bibitem{boxARIMA}
George Edward~Pelham Box and Gwilym Jenkins.
\newblock {\em Time Series Analysis, Forecasting and Control}.
\newblock Holden-Day, Inc., USA, 1990.

\bibitem{boyd2018introduction}
Stephen Boyd and Lieven Vandenberghe.
\newblock {\em Introduction to applied linear algebra: vectors, matrices, and
  least squares}.
\newblock Cambridge university press, 2018.

\bibitem{brumm2017adaptive}
Johannes Brumm and Simon Scheidegger.
\newblock Using adaptive sparse grids to solve high-dimensional dynamic models.
\newblock {\em Econometrica}, 85(5):1575--1612, 2017.

\bibitem{bungartz2004sparse}
Hans-Joachim Bungartz and Michael Griebel.
\newblock Sparse grids.
\newblock {\em Acta Numerica}, 13:1--123, 2004.

\bibitem{castillo2023benefits}
Juan~Camilo Castillo.
\newblock Who benefits from surge pricing?
\newblock {\em Available at SSRN 3245533}, 2023.

\bibitem{castillo2017surge}
Juan~Camilo Castillo, Dan Knoepfle, and Glen Weyl.
\newblock Surge pricing solves the wild goose chase.
\newblock In {\em Proceedings of the 2017 ACM Conference on Economics and
  Computation}, pages 241--242, 2017.

\bibitem{chen2023tsmixer}
Si-An Chen, Chun-Liang Li, Nate Yoder, Sercan~O Arik, and Tomas Pfister.
\newblock Tsmixer: An all-mlp architecture for time series forecasting.
\newblock {\em arXiv preprint arXiv:2303.06053}, 2023.

\bibitem{chernozhukov2018double}
Victor Chernozhukov, Denis Chetverikov, Mert Demirer, Esther Duflo, Christian
  Hansen, Whitney Newey, and James Robins.
\newblock Double/debiased machine learning for treatment and structural
  parameters, 2018.

\bibitem{de1972calculating}
Carl De~Boor.
\newblock On calculating with b-splines.
\newblock {\em Journal of Approximation theory}, 6(1):50--62, 1972.

\bibitem{de1978practical}
Carl De~Boor and Carl De~Boor.
\newblock {\em A practical guide to splines}, volume~27.
\newblock springer-verlag New York, 1978.

\bibitem{du2019improve}
Shuyang Du, James Lee, and Farzin Ghaffarizadeh.
\newblock Improve user retention with causal learning.
\newblock In {\em The 2019 ACM SIGKDD Workshop on Causal Discovery}, pages
  34--49. PMLR, 2019.

\bibitem{eilers1996flexible}
Paul~HC Eilers and Brian~D Marx.
\newblock Flexible smoothing with b-splines and penalties.
\newblock {\em Statistical science}, 11(2):89--121, 1996.

\bibitem{griebel1991parallelizable}
Michael Griebel.
\newblock A parallelizable and vectorizable multi-level algorithm on sparse
  grids.
\newblock In {\em Parallel algorithms for partial differential equations},
  1991.

\bibitem{griebel1992combination}
Michael Griebel, Michael Schneider, and Christoph Zenger.
\newblock A combination technique for the solution of sparse grid problems.
\newblock In {\em Iterative Methods in Linear Algebra}, page 263–281, 1992.

\bibitem{hahn2020bayesian}
P~Richard Hahn, Jared~S Murray, and Carlos~M Carvalho.
\newblock Bayesian regression tree models for causal inference: Regularization,
  confounding, and heterogeneous effects (with discussion).
\newblock {\em Bayesian Analysis}, 15(3):965--1056, 2020.

\bibitem{haykin1994neural}
Simon Haykin.
\newblock {\em Neural networks: a comprehensive foundation}.
\newblock Prentice Hall PTR, 1994.

\bibitem{he2015deep}
Kaiming He, Xiangyu Zhang, Shaoqing Ren, and Jian Sun.
\newblock Deep residual learning for image recognition, 2015.

\bibitem{hellinger1909neue}
Ernst Hellinger.
\newblock Neue begr{\"u}ndung der theorie quadratischer formen von
  unendlichvielen ver{\"a}nderlichen.
\newblock {\em Journal f{\"u}r die reine und angewandte Mathematik},
  1909(136):210--271, 1909.

\bibitem{doi:10.1177/0149206320960533}
Aaron~D. Hill, Scott~G. Johnson, Lindsey~M. Greco, Ernest~H. O’Boyle, and
  Sheryl~L. Walter.
\newblock Endogeneity: A review and agenda for the methodology-practice divide
  affecting micro and macro research.
\newblock {\em Journal of Management}, 47(1):105--143, 2021.

\bibitem{hill2011bayesian}
Jennifer~L Hill.
\newblock Bayesian nonparametric modeling for causal inference.
\newblock {\em Journal of Computational and Graphical Statistics},
  20(1):217--240, 2011.

\bibitem{hu2022average}
Yuchen Hu, Shuangning Li, and Stefan Wager.
\newblock Average direct and indirect causal effects under interference.
\newblock {\em Biometrika}, 109(4):1165--1172, 2022.

\bibitem{imai2013estimating}
Kosuke Imai and Marc Ratkovic.
\newblock Estimating treatment effect heterogeneity in randomized program
  evaluation.
\newblock 2013.

\bibitem{kunzel2019metalearners}
S{\"o}ren~R K{\"u}nzel, Jasjeet~S Sekhon, Peter~J Bickel, and Bin Yu.
\newblock Metalearners for estimating heterogeneous treatment effects using
  machine learning.
\newblock {\em Proceedings of the national academy of sciences},
  116(10):4156--4165, 2019.

\bibitem{li2019enhancing}
Shiyang Li, Xiaoyong Jin, Yao Xuan, Xiyou Zhou, Wenhu Chen, Yu-Xiang Wang, and
  Xifeng Yan.
\newblock Enhancing the locality and breaking the memory bottleneck of
  transformer on time series forecasting.
\newblock {\em Advances in neural information processing systems}, 32, 2019.

\bibitem{lim2021temporal}
Bryan Lim, Sercan~{\"O} Ar{\i}k, Nicolas Loeff, and Tomas Pfister.
\newblock Temporal fusion transformers for interpretable multi-horizon time
  series forecasting.
\newblock {\em International Journal of Forecasting}, 37(4):1748--1764, 2021.

\bibitem{10.1145/3357384.3357835}
Ziqi Liu, Dong Wang, Qianyu Yu, Zhiqiang Zhang, Yue Shen, Jian Ma, Wenliang
  Zhong, Jinjie Gu, Jun Zhou, Shuang Yang, and Yuan Qi.
\newblock Graph representation learning for merchant incentive optimization in
  mobile payment marketing.
\newblock In {\em Proceedings of the 28th ACM International Conference on
  Information and Knowledge Management}, CIKM '19, page 2577–2584, New York,
  NY, USA, 2019. Association for Computing Machinery.

\bibitem{moritz2018ray}
Philipp Moritz, Robert Nishihara, Stephanie Wang, Alexey Tumanov, Richard Liaw,
  Eric Liang, Melih Elibol, Zongheng Yang, William Paul, Michael~I Jordan,
  et~al.
\newblock Ray: A distributed framework for emerging $\{$AI$\}$ applications.
\newblock In {\em 13th USENIX symposium on operating systems design and
  implementation (OSDI 18)}, pages 561--577, 2018.

\bibitem{nesterov1998primal}
Yu~E Nesterov and Michael~J Todd.
\newblock Primal-dual interior-point methods for self-scaled cones.
\newblock {\em SIAM Journal on optimization}, 8(2):324--364, 1998.

\bibitem{nie2021quasi}
Xinkun Nie and Stefan Wager.
\newblock Quasi-oracle estimation of heterogeneous treatment effects.
\newblock {\em Biometrika}, 108(2):299--319, 2021.

\bibitem{powers2018some}
Scott Powers, Junyang Qian, Kenneth Jung, Alejandro Schuler, Nigam~H Shah,
  Trevor Hastie, and Robert Tibshirani.
\newblock Some methods for heterogeneous treatment effect estimation in high
  dimensions.
\newblock {\em Statistics in medicine}, 37(11):1767--1787, 2018.

\bibitem{pya2015shape}
Natalya Pya and Simon~N Wood.
\newblock Shape constrained additive models.
\newblock {\em Statistics and computing}, 25:543--559, 2015.

\bibitem{schaab2022adaptive}
Andreas Schaab and Allen~T Zhang.
\newblock Dynamic programming in continuous time with adaptive sparse grids.
\newblock {\em Working Paper}, 2022.

\bibitem{shalit2017estimating}
Uri Shalit, Fredrik~D Johansson, and David Sontag.
\newblock Estimating individual treatment effect: generalization bounds and
  algorithms.
\newblock In {\em International conference on machine learning}, pages
  3076--3085. PMLR, 2017.

\bibitem{smolyak1963quadrature}
Sergei~Abramovich Smolyak.
\newblock Quadrature and interpolation formulas for tensor products of certain
  classes of functions.
\newblock {\em Doklady Akademii Nauk}, 148:1042–1045, 1963.

\bibitem{strutz2011data}
Tilo Strutz.
\newblock {\em Data fitting and uncertainty: A practical introduction to
  weighted least squares and beyond}, volume~1.
\newblock Springer, 2011.

\bibitem{VaswaniSPUJGKP17}
Ashish Vaswani, Noam Shazeer, Niki Parmar, Jakob Uszkoreit, Llion Jones,
  Aidan~N. Gomez, Lukasz Kaiser, and Illia Polosukhin.
\newblock Attention is all you need.
\newblock {\em CoRR}, abs/1706.03762, 2017.

\bibitem{wang2019global}
Yu~Wang, Wotao Yin, and Jinshan Zeng.
\newblock Global convergence of admm in nonconvex nonsmooth optimization.
\newblock {\em Journal of Scientific Computing}, 78:29--63, 2019.

\bibitem{Yu_Wu_Cai_Liu_Zhang_Gu_Zeng_Gu_2021}
Li~Yu, Zhengwei Wu, Tianchi Cai, Ziqi Liu, Zhiqiang Zhang, Lihong Gu, Xiaodong
  Zeng, and Jinjie Gu.
\newblock Joint incentive optimization of customer and merchant in mobile
  payment marketing.
\newblock {\em Proceedings of the AAAI Conference on Artificial Intelligence},
  35(17):15000--15007, May 2021.

\bibitem{zeng2023transformers}
Ailing Zeng, Muxi Chen, Lei Zhang, and Qiang Xu.
\newblock Are transformers effective for time series forecasting?
\newblock In {\em Proceedings of the AAAI conference on artificial
  intelligence}, volume~37, pages 11121--11128, 2023.

\bibitem{zenger1991sparse}
Christoph Zenger.
\newblock Sparse grids.
\newblock In {\em Proceedings of the Research Workshop of the Israel Science
  Foundation on Multiscale Phenomenon, Modelling and Computation}, page~86,
  1991.

\bibitem{zhou2021informer}
Haoyi Zhou, Shanghang Zhang, Jieqi Peng, Shuai Zhang, Jianxin Li, Hui Xiong,
  and Wancai Zhang.
\newblock Informer: Beyond efficient transformer for long sequence time-series
  forecasting.
\newblock In {\em Proceedings of the AAAI conference on artificial
  intelligence}, volume~35, pages 11106--11115, 2021.

\bibitem{zhou2022film}
Tian Zhou, Ziqing Ma, Qingsong Wen, Liang Sun, Tao Yao, Wotao Yin, Rong Jin,
  et~al.
\newblock Film: Frequency improved legendre memory model for long-term time
  series forecasting.
\newblock {\em Advances in Neural Information Processing Systems},
  35:12677--12690, 2022.

\bibitem{zhou2022fedformer}
Tian Zhou, Ziqing Ma, Qingsong Wen, Xue Wang, Liang Sun, and Rong Jin.
\newblock Fedformer: Frequency enhanced decomposed transformer for long-term
  series forecasting.
\newblock In {\em International Conference on Machine Learning}, pages
  27268--27286. PMLR, 2022.

\end{thebibliography}
\bibliographystyle{plain}

\appendix
\section{System Overview}
\label{sec:appendix_system}
An end-to-end view containing all components of the budget allocation system described in this paper can be found in Fig.~\ref{fig:architecture}.

\begin{figure*}[htbp]
\centering
\includegraphics[width=0.85\textwidth]{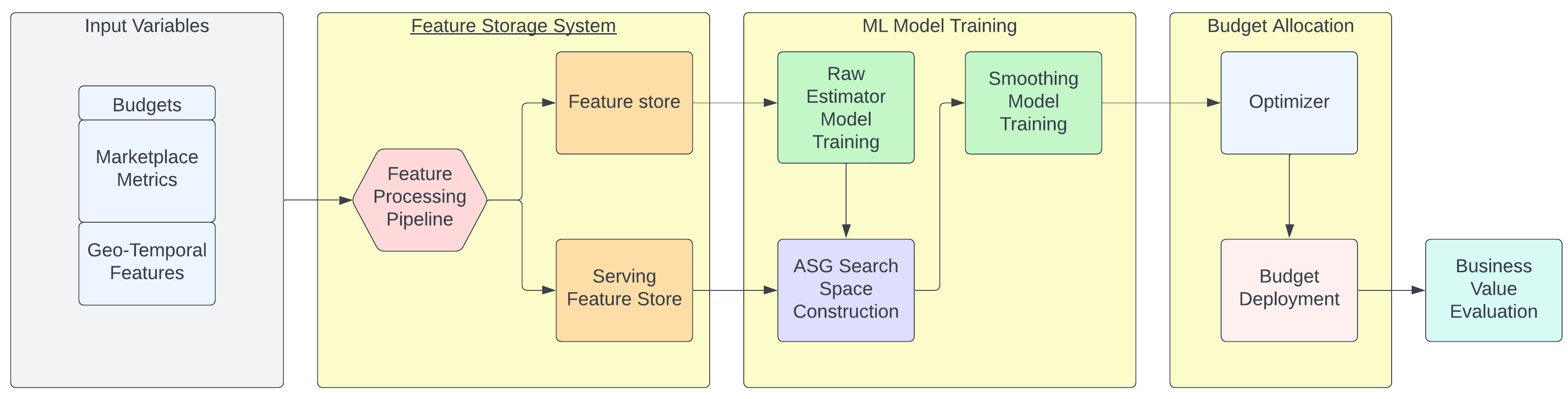}
\caption{Allocation system architecture: raw data processing into feature store, ML model training and serving, optimization, and evaluation.}
\label{fig:architecture}
\end{figure*}

\section{Metalearner Comparison}  \label{sec:appendix_estimator}
In Table~\ref{tab:comparison}, we evaluate the predicted incremental supply efficiency of
the S-learner with percentage error (BIAS) at the city-week granularity, which shows closer
alignment with experimental ground truth compared to the R-learner.
The untuned S-learner model exhibits a median BIAS similar to the current baseline performance.
Unlike the baseline, the S-learner does not consistently underestimate lever efficiency,
though certain weeks display significant prediction errors.

\begin{table*}[h!]
    \caption{Comparison of S-Learner and R-Learner Incremental Supply Efficiency Predictions}
    \label{tab:comparison}
    \centering
    \begin{tabularx}{\textwidth}{ccccccc}
        \toprule
        \textbf{Week} & \textbf{Experiment} & \textbf{Baseline Prediction} & \textbf{S-Learner Prediction} & \textbf{S-Learner BIAS} & \textbf{R-Learner Prediction} & \textbf{R-Learner BIAS} \\
        \midrule
        week 1 & 0.029 & 0.012 & 0.087 & 0.058 & 0.193 & 0.501 \\
        week 2  & 0.045 & 0.010 & 0.017 & -0.029 & 0.314 & 0.125 \\
        week 3  & 0.053 & 0.010 & 0.349 & 0.295 & 0.113 & 0.720 \\
        week 4 & 0.045 & 0.008 & 0.621 & 0.576 & 0.314 & 0.478 \\
        week 5 & 0.027 & 0.009 & 0.104 & 0.077 & 0.368 & 0.132 \\
        week 6 & 0.045 & 0.008 & 0.026 & -0.019 & 0.385 & 0.086 \\
        week 7  & 0.053 & 0.013 & 0.075 & 0.022 & 0.660 & 0.387 \\
        week 8 & 0.046 & 0.016 & 0.015 & -0.030 & 0.182 & 0.614 \\
        \bottomrule
    \end{tabularx}
\end{table*}

\section{B-spline Proofs}  \label{sec:appendix}
In this section, we discuss the monotonicity and concavity conditions of B-spline functions in more details and provide some proofs.

\begin{proposition} \label{prop:1stordernecessary}
    Given a univariate B-spline function
    \begin{align}
        S(x)=\sum_{i=0}^{n-1}a_{i}B_i^r(x;\boldsymbol{t})
    \end{align}
    with knots $\boldsymbol{t}=\{t_0,\dots,t_{n+1}\}$,
    the sufficient condition for $S(x)$ to be monotonically non-decreasing over $[t_r,t_{n}]$ is
    \begin{align}\label{eq:monotonicity_sufficient}
        a_{i}-a_{i-1}\geq 0, \ i=1,\dots,n-1
    \end{align}
\end{proposition}

\begin{proof}
    Following \cite{de1978practical,de1972calculating}, the first order derivative of $S(x)$ with respect to $x$ can be written as
    \begin{align} \label{eq:bsplinefirstorder}
        \frac{\partial S(x)}{\partial x}=\sum_{i=1}^{n-1} \frac{r}{\Delta_i}(a_i-a_{i-1})B_i^{r-1}(x;\boldsymbol{t^{\prime}})
    \end{align}
    where  $\Delta_{i} = t_{i+r}-t_{i}$.

    Note that $\frac{\partial S(x)}{\partial x}$ is a new B-spline function with degree $r-1$ and knots $\boldsymbol{t^{\prime}}=\boldsymbol{t}[1\colon -1]$.
    Since by definition (Equations \eqref{eq:0thbasis} and \eqref{eq:higherbasis}), B-spline basis $B_i^{r-1}(x;\boldsymbol{t^{\prime}})$ cannot be negative.
    Given $\Delta_{i}\geq 0$, it is clear that if $a_{i}-a_{i-1}\geq 0$ for $i=1,\dots,n-1$, then $\frac{\partial S(x)}{\partial x} \geq 0$
\end{proof}

\begin{proposition}\label{prop:2ndordernecessary}
    Given a univariate B-spline function
    \begin{align}
        S(x)=\sum_{i=0}^{n-1}a_{i}B_i^r(x;\boldsymbol{t})
    \end{align}
    with knots $\boldsymbol{t}=\{t_0,\dots,t_{n+1}\}$,
    the sufficient condition for $S(x)$ to be concave over $[t_r,t_{n}]$ is
    \begin{align}\label{eq:concavity_sufficient}
        &  \frac{a_{i}-a_{i-1}}{\Delta_{i}} -\frac{a_{i-1}-a_{i-2}}{\Delta_{i-1}}   \leq 0,  \ i=2,\dots,n-1
    \end{align}
    where $\Delta_{i} = t_{i+r}-t_{i}$
\end{proposition}

\begin{proof}
    Based on Equation \eqref{eq:bsplinefirstorder}, the second order derivative of $S(x)$ with respect to $x$ can be written as
    \begin{align} \label{eq:bsplinesecondorder}
    & & \frac{\partial^2 S(x)}{\partial x^2}=\sum_{i=2}^{n-1} \frac{r-1}{t_{i+r-1}-t_i}[\frac{r}{\Delta_i}(a_i-a_{i-1}) \notag \\ & & -\frac{r}{\Delta_{i-1}}(a_{i-1}-a_{i-2})] B_i^{r-2}(x;\boldsymbol{t^{\prime\prime}})
    \end{align}
    where  $\Delta_{i} = t_{i+r}-t_{i}$.

    Note that $\frac{\partial^2 S(x)}{\partial x^2}$ is a new B-spline function with degree $r-2$ and knots $\boldsymbol{t^{\prime\prime}}=\boldsymbol{t}[2 \colon -2]$.

    Similarly, since B-spline basis $B_i^{r-2}(x;\boldsymbol{t^{\prime\prime}})$ and $t_{i+r-1}-t_i$ are non-negative, it is clear that if $\frac{r}{\Delta_i}(a_i-a_{i-1})-\frac{r}{\Delta_{i-1}}(a_{i-1}-a_{i-2}) \leq 0$ for $i=2,\dots,n-1$, then $\frac{\partial^2 S(x)}{\partial x^2} \leq 0$

\end{proof}

\begin{proposition}\label{prop:1stordersufficient}
    Given a univariate B-spline function
    \begin{align}
        S(x)=\sum_{i=0}^{n-1}a_{i}B_i^r(x;\boldsymbol{t})
    \end{align}
    with knots $\boldsymbol{t}=\{t_0,\dots,t_{n+1}\}$,
    when the basis function degree $r=2$, Equation \eqref{eq:monotonicity_sufficient} is also a necessary condition for $S(x)$ to be monotonically non-decreasing over $[t_2,t_{n}]$
\end{proposition}

\begin{proof}
    When $r=2$, the first order derivative in Equation \eqref{eq:bsplinefirstorder} can be written as
    \begin{align}
    \frac{\partial S(x)}{\partial x}=\sum_{i=1}^{n-1} \frac{2}{\Delta_i}(a_i-a_{i-1})B_i^{1}(x;\boldsymbol{t^{\prime}})
    \end{align}
    where  $\Delta_{i} = t_{i+2}-t_{i}$.

    By Lemma \ref{lemma:1degreebasis}, when $x \in [t_i, t_{i+1})$, $ \frac{\partial S(x)}{\partial x}$ can be written as

    \begin{align}
        \frac{\partial S(x)}{\partial x}= \frac{2}{\Delta_i}(a_i-a_{i-1})\frac{x - t_{i}}{t_{i+1} - t_{i}}
    \end{align}

    It is clear that $\frac{\partial S(x)}{\partial x}\geq 0$ only if $a_i-a_{i-1}\geq 0$. It is straightforward to generalize this for all $x \in [t_2, t_n]$

\end{proof}

\begin{lemma} \label{lemma:1degreebasis}
    The degree 1 B-spline basis function $B_i^{1}(x;\boldsymbol{t})=\frac{x - t_{i}}{t_{i+1} - t_{i}}$ for $x \in [t_i, t_{i+1})$
\end{lemma}
\begin{proof}
    Based on Equation \eqref{eq:0thbasis}, when $x \notin [t_i, t_{i+1})$, we have $B_i^{0}(x;\boldsymbol{t})=0$

    Then, based on Equation \eqref{eq:higherbasis}, we have
    \begin{align}
        B_{i}^{1}(x;\boldsymbol{t}) =
        \begin{cases}
            \frac{x - t_{i}}{t_{i+1} - t_{i}} & \text{when } x \in [t_i, t_{i+1}) \\
            0                                 & \text{otherwise }
        \end{cases}
    \end{align}
\end{proof}

\begin{proposition} \label{prop:2ndordersufficient}
    Given a univariate B-spline function
    \begin{align}
        S(x)=\sum_{i=0}^{n-1}a_{i}B_i^r(x;\boldsymbol{t})
    \end{align}
    with knots $\boldsymbol{t}=\{t_0,\dots,t_{n+1}\}$,
    when the basis function degree $r=2$, Equation \eqref{eq:concavity_sufficient} is also a necessary condition for $S(x)$ to be concave over $[t_2,t_{n}]$
\end{proposition}

\begin{proof}
    When $r=2$, the second order derivative in Equation \eqref{eq:bsplinesecondorder} can be written as
    \begin{align}
    & & \frac{\partial^2 S(x)}{\partial x^2}=\sum_{i=2}^{n-1} \frac{1}{t_{i+1}-t_i}[\frac{2}{\Delta_i}(a_i-a_{i-1})\\ & & -\frac{2}{\Delta_{i-1}}(a_{i-1}-a_{i-2})] B_i^{0}(x;\boldsymbol{t^{\prime\prime}})
    \end{align}
    where  $\Delta_{i} = t_{i+2}-t_{i}$ and knots $\boldsymbol{t^{\prime\prime}}=\boldsymbol{t}[2 \colon -2]$..

    Based on Equation \eqref{eq:0thbasis}, for any $t_{2+i} \leq x < t_{2+i+1}, i=0,\dots,n-1$, $\frac{\partial^2 S(x)}{\partial x^2}$ is equal to
    \begin{align}
        \frac{\partial^2 S(x)}{\partial x^2} = \frac{2}{\Delta_i}(a_i-a_{i-1}) -\frac{2}{\Delta_{i-1}}(a_{i-1}-a_{i-2})
    \end{align}
    because other bases are 0.

    It is clear that $\frac{\partial^2 S(x)}{\partial x^2} \leq 0 $ over $[t_2,t_n]$ only if $\frac{2}{\Delta_i}(a_i-a_{i-1}) -\frac{2}{\Delta_{i-1}}(a_{i-1}-a_{i-2}) \leq 0, \forall i $

\end{proof}

\begin{proposition}
\Cref{prop:1stordernecessary,prop:2ndordernecessary} can be generalized to any multivariate B-spline function.
\end{proposition}

\begin{proof}
    A multivariate B-spline function can be written as
    \begin{align}
        S(x_1,\dots,x_N)=\sum_{i_1=0}^{n_1-1} \cdots \sum_{i_N=0}^{n_N-1} a_{i_1\dots i_N} \prod_{l=1}^{N} B_{i_l}^{r_l}(x_l;\boldsymbol{t_l})
    \end{align}
    Without loss of generality, assume we are interested in its partial derivative with respect to $x_j$. $S(x_1,\dots,x_N)$ can be re-arranged as follows to separate items involving $x_j$ and others
    \begin{align}
        S(x_1,\dots,x_N)=\sum_{i_1=0}^{n_1-1} \cdots \sum_{i_N=0}^{n_N-1}\prod_{l=1,l\neq j}^{N} B_{i_l}^{r_l}(x_l;\boldsymbol{t_l})\sum_{i_j=0}^{n_j-1}a_{i_1\dots i_N}B_{i_j}^{r_j}(x_j;\boldsymbol{t_j})
    \end{align}

    Denote
    \begin{align}
        C(x_{l\neq j}|i_{l\neq j})=\prod_{l=1,l\neq j}^{N} B_{i_l}^{r_l}(x_l;\boldsymbol{t_l})
    \end{align}
    it is always non-negative because it is a product of multiple B-spline bases and each B-spline basis cannot be negative by definition.

    Denote
    \begin{align}
        S(x_j|i_{l\neq j})=\sum_{i_j=0}^{n_j-1}a_{i_1\dots i_N}B_{i_j}^{r_j}(x_j;\boldsymbol{t_j})
    \end{align}
    it is a univariate B-spline function.

    $S(x_1,\dots,x_N)$ can be written as
    \begin{align}
        S(x_1,\dots,x_N)=\sum_{i_1=0}^{n_1-1} \cdots \sum_{i_N=0}^{n_N-1}C(x_{l\neq j}|i_{l\neq j})S(x_j|i_{l\neq j})
    \end{align}

    The derivative of $S(x_1,\dots,x_N)$ with respect to $x_j$ is equal to a sum of non-negative $C(x_{l\neq j}|i_{l\neq j})$ times the derivative of univariate B-spline $S(x_j|i_{l\neq j})$ with respect to $x_j$.
    Hence, \Cref{prop:1stordernecessary,prop:2ndordernecessary} for a univariate B-spline function can be extended to a multivariate B-spline function.
\end{proof}

\section{Optimizer Implementation}  \label{sec:appendix_optimizer}
In order to solve the per-city problems in the $\mathbf{b}$-update step, we use \texttt{cvxopt}~\cite{
    andersen2020cvxopt}, which is a primal-dual interior-point convex solver that uses Nesterov-Todd scaling~\cite{nesterov1998primal}. Since the modified problem is both city-independent and convex, this inner solver works well for us.

A major challenge to solving the problem in a single thread, instead of in parallel, is that the constraint complexity is high, and collinear constraints can affect the Cholesky decomposition in the Nesterov-Todd scaling step. The single city constraints are much easier to manage and therefore never cause this problem in the \texttt{cvxopt} step. Regardless, managing constraints, including online validations when receiving constraints from client teams remains a key practical effort of the system.

Recall that our constraints that cross multiple cities are equality constraints. Therefore, $g(z)$ in Equation~\ref{eqn:admm_formulation} is essentially an indicator function that is zero-valued when the constraint is met and takes a large value when the constraint is not met.

The update $\mathbf{z}$ step can therefore be translated to a simple quadratic optimization problem. We can then use
the Cauchy-Schwartz inequality to simplify further, eventually resulting in an algebraic expression, that for the
constraint in Equation~\ref{eqn:multiweek_problem}, results in:
\begin{equation}
\label{eqn:cs_zstep}
\mathbf{z}^{k+1}_c := \frac{B-  \mathbf{e} \cdot \sum\limits_{c \in \textrm{cities}} \mathbf{u}^k_c}{|C| |L|} \mathbf{e} + \mathbf{u}^k_c,
\end{equation}
where $\mathbf{u}^k_c := \mathbf{y}^k_c + \mathbf{b}^{k+1}_c$ and $\mathbf{e}$ is a vector of all ones. This simplification allows the update $\mathbf{z}$ step to be calculated algebraically.

Several practical improvements allow the system to scale well and perform with high stability in production:
\begin{itemize}
    \item \textbf{Horizontal scaling with Ray}.~\cite{moritz2018ray} Because the update x step is handling each city individually, we can run optimization for each city in parallel. By leveraging the Ray framework, we achieved approximately a 50x speedup per iteration in that step.
    \item \textbf{Adaptive} $\mathbf{\rho}$. In order to improve stability of the system, we adaptively increase $\rho$ when non-convex sub-problems are detected. As mentioned above, this modifies the eigenvalues of the Hessian to ensure convexness.
    \item \textbf{Early Stopping \& Optimal criteria}. The ADMM algorithm can quickly get to a point where it is close to the global optima, but it takes a long time to provide a solution with multi-digit accuracy. Our implementation requires a compromise between runtime and accuracy. We define a solution is optimal if:
    \begin{itemize}
        \item Constraints are obeyed within 0.1\% accuracy
        \item All sub problems (single city) are converged
        \item The change in objective value is small from the previous iteration
        \item Primal and dual feasibility are satisfied
    \end{itemize}
\end{itemize}

\end{document}